\documentclass{article} 
\usepackage{iclr2026_conference,times}

\usepackage{amsmath,amsfonts,bm}

\def\eqref#1{equation~\ref{#1}}

\def\1{\bm{1}}

\DeclareMathAlphabet{\mathsfit}{\encodingdefault}{\sfdefault}{m}{sl}
\SetMathAlphabet{\mathsfit}{bold}{\encodingdefault}{\sfdefault}{bx}{n}

\newcommand{\R}{\mathbb{R}}

\usepackage{hyperref}
\usepackage{url}

\usepackage{booktabs} 
\usepackage{microtype}
\usepackage{graphicx}
\usepackage{ulem} 

\usepackage{amsmath}
\usepackage{amssymb}
\usepackage{mathtools}
\usepackage{amsthm}
\usepackage[T1]{fontenc}     
\usepackage{subcaption}
\usepackage{makecell}  
\usepackage[utf8]{inputenc}  

\usepackage{etoolbox} 

\makeatletter
\patchcmd{\@footnotetext}{\Hy@raisedlink{\hyper@anchorstart{footnote\@thefnmark}\hyper@anchorend}}{}{}{}
\renewcommand{\And}{\end{tabular}\hspace{0.75em}\begin{tabular}[t]{c}}

\makeatother

\theoremstyle{plain}
\newtheorem{theorem}{Theorem}[section]

\newtheorem{lemma}[theorem]{Lemma}

\theoremstyle{definition}

\theoremstyle{remark}

\definecolor{coolpurple}{rgb}{0.2, 0.6, 0.2}
\definecolor{darkgreen}{rgb}{0.0, 0.6, 0.0}

\newcommand{\blfootnote}[1]{%
  \begingroup
  \renewcommand\thefootnote{}\footnotetext{#1}%
  \endgroup
}

\title{$\mu$-Parameterization for Mixture of Experts}

\author{
Jan Małaśnicki $^*$ \\ University of Warsaw \\
\And
Kamil Ciebiera $^*$ \\ University of Warsaw \\
\And
Mateusz Boruń \\University of Warsaw \\Syntro \\
\AND
Maciej Pióro \\IPPT PAN $^1$ \\University of Warsaw \\IDEAS NCBR \\
\And
Jan Ludziejewski \\University of Warsaw \\IDEAS NCBR \\
\And
Maciej Stefaniak \\University of Warsaw \\
\AND
Michał Krutul \\University of Warsaw \\IDEAS NCBR \\
\And
Sebastian Jaszczur \\University of Warsaw \\IDEAS NCBR \\
\And
Marek Cygan \\University of Warsaw \\Nomagic \\
\AND
Kamil Adamczewski \\IDEAS NCBR \\Wrocław University $^2$ \\
\And
Jakub Krajewski \\
University of Warsaw \\IDEAS NCBR \\
\AND
}

\iclrfinalcopy 
\begin{document}

\maketitle

\blfootnote{$^{1}$ Institute of Fundamental Technological Research, Polish Academy of Sciences,
    $^{2}$ Wroclaw University of Science and Technology,
    Correspondence to:
    Jan Małaśnicki <jan.malasnicki@gmail.com>,
    Kamil Ciebiera <kp.ciebiera@student.uw.edu.pl>,
    Jakub Krajewski <gim.jakubk@gmail.com>
}

\vspace{-1cm}
\begin{abstract}
Recent years have seen a growing interest and adoption of LLMs, with Mixture-of-Experts (MoE) emerging as a leading architecture in extremely large models. Currently, the largest open-source models reach over $1$T parameters. At such scales, hyperparameter tuning becomes prohibitively expensive. Precisely for this reason, the $\mu$Transfer is becoming a key technique. It allows for seamless transfer of optimal hyperparameters across model scales, resulting in a huge reduction in tuning costs. However, existing work has primarily focused on dense LLMs, leaving MoE architectures unexplored. In this work, we derive a $\mu$-Parameterization for MoE, providing theoretical guarantees for feature learning across model widths. Our experiments demonstrate that the optimal learning rate reliably transfers across model sizes, establishing a foundation for efficient hyperparameter tuning in large-scale MoE models.
\end{abstract}

\section{Introduction}
\label{s:intro}

Scaling large language models (LLMs) has been a central driver of recent progress in AI. Scaling laws show that larger models trained for longer consistently achieve better performance. This naturally places us in a regime where computational resources are perpetually the limiting factor. In such settings, the efficiency of hyperparameter tuning is as important as the efficiency of the final training run, since both draw from the same compute budget.

The $\mu$-Parameterization ($\mu$P) \citep{yang2021feature} provides a principled framework for stable and predictable training dynamics across model widths. By reparameterizing models to preserve feature learning in the infinite-width limit, $\mu$P ensures that optimal hyperparameters (e.g., learning rates, initialization scales) remain invariant to model width. This allows hyperparameter searches to be conducted on small models, saving most of the computational resources for the final training run.

In parallel, Mixture-of-Experts (MoE) has become one of the most widely adopted architectural innovations for improving the model efficiency \citep{deepseekai2025deepseekv3technicalreport, openai2025gptoss120bgptoss20bmodel, kimiteam2025kimik2openagentic}. MoE decouples parameter count from computational cost through conditional computation, replacing each dense feedforward block with a router and a set of expert networks. Despite its popularity and effectiveness, MoE has not yet been parameterized for feature learning in the infinite-width limit, leaving open the question of whether $\mu$P can be extended to sparse architectures of this kind. 

In this work, we extend $\mu$P to MoE architectures, providing both theoretical grounding and empirical validation. Our main contributions are:

\begin{itemize}
\setlength{\itemsep}{4pt}
\setlength{\parskip}{0pt}
\setlength{\parsep}{0pt}



\item \textbf{Parameterization with learning rate transfer in MoE models.} We find a parameterization for MoE models that transfers the optimal learning rate across model widths. We empirically verify this transferability.

\item \textbf{Theoretical grounding of $\mu$P for MoE.} Building on \citet{yang2022tuning}, we provide a theoretically principled derivation of a parameterization that ensures feature learning is preserved across all weights within MoE models.

\item \textbf{Limits of transferability.} We show that while learning rate transfer holds across widths, it breaks down when varying the granularity of MoE, identifying boundaries for hyperparameter transfer in MoE.

\end{itemize}

\section{Background and Related Work}
\label{s:rel}

\textbf{Mixture of Experts.} Mixture of Experts was originally introduced by \citep{moe1991}, and later proposed in the context of language modeling by \citep{shazeer2017outrageously}. This approach has since been successfully integrated into the Transformer \citep{vaswani2017attention} architecture in multiple works, including \citep{fedus2022switch, lepikhin2020gshard, du2022glam, zhou2022mixtureofexperts}. 
In a Transformer model, the MoE layer is typically constructed by replacing the Feed-Forward component with a set of \textit{experts}. Each expert retains the design of the original Feed-Forward layer, consisting of two linear layers with a non-linearity between them. Crucially, for any given input token, only a fraction of these experts are activated. The selection of experts for each token is determined by a routing mechanism - a simple linear layer followed by a softmax normalization and a Top-$k$ choice.

In the standard Switch layer \citep{fedus2022switch}, each of the experts is of the same size as in the corresponding dense Transformer. This assumption is relaxed in fine-grained MoE \citep{dai2024deepseekmoe, ludziejewskiscaling}, where for granularity $G$, the hidden size of each expert is reduced by a factor of $G$, while the number of experts and the router's top-$k$ value are both multiplied by $G$. In this way, the model has greater flexibility in mapping tokens to experts, while the total number and activated parameters remain approximately constant.

\textbf{Zero-shot hyperparameter transfer.} Standard Parametrization (SP) initializes weights with the usual variance scaling (like Xavier or Kaiming style: variance $\propto 1/d_{\text{in}}$ for many layers). SP learning rates do not scale with width, which means that as the model width increases, gradient magnitudes and update sizes can change unpredictably. Thus, SP often fails to preserve stability and hyperparameter transfer in scaling neural networks. To overcome this limitation, \citep{yang2021feature} introduced Maximal Update Parameterization ($\mu$-Parameterization, or $\mu$P). $\mu$P ensures that each layer in a network receives updates of the same order of magnitude during training, regardless of width. This allows for what is known as the feature learning regime, where internal representations evolve in a meaningful way as training progresses. Crucially, $\mu$P enables hyperparameter transfer across model sizes: one can tune learning rates and initialization on a small model and zero-shot transfer them to a large one, as shown empirically and theoretically in \citet{yang2022tuning}. This paradigm, called $\mu$Transfer, has been shown to dramatically reduce the cost of training large models while maintaining performance.
Later works \citep{spectral_condition, scaling_exponents} reformulate and generalize $\mu$P theory, while \citet{completeP} and \citet{tp_6} include transfer across model depths.  
Despite its success on many architectures such as Transformers and ResNets, extending $\mu$P to Mixture-of-Experts models remains an open challenge. In this work, we address this problem.

\begin{table*}
    \centering
    \begin{tabular}{l|ccc}
    \toprule
    \textbf{Component} & \textbf{Init. Var.} & \textbf{Multiplier} & \textbf{LR (Adam)} \\
    \midrule
    Embedding & \textcolor{gray}{$1.0$} & \textcolor{gray}{$1.0$} & \textcolor{gray}{$1.0$} \\
    Unembedding & \textcolor{blue}{$1.0$} & \textcolor{blue}{$1/d_\text{input}$} & \textcolor{gray}{$1.0$} \\
    \makecell[l]{Attention \\(Q, K, V, O)} & \textcolor{gray}{$1/d_\text{input}$} & \textcolor{gray}{$1.0$} & \textcolor{blue}{$1/d_\text{input}$} \\
    \makecell[l]{Feed-forward \\(dense)} & \textcolor{gray}{$1/d_\text{input}$} & \textcolor{gray}{$1.0$} & \textcolor{blue}{$1/d_\text{input}$} \\
    \midrule
    Experts (MoE) & \textcolor{gray}{$1/d_\text{input}$} & \textcolor{gray}{$1.0$} & \textcolor{red}{$1/d_\text{input}$} \textcolor{gray}{$|$} \textcolor{darkgreen}{$1/d_\text{input}$} \\
    Router (MoE) & \textcolor{gray}{$1/d_\text{input}$} \textcolor{gray}{$|$} \textcolor{darkgreen}{$1.0$} & \textcolor{gray}{$1.0$} \textcolor{gray}{$|$} \textcolor{darkgreen}{$1/d_\text{input}$} & \textcolor{gray}{$1.0$} \\
    \bottomrule
    \end{tabular}
    \vspace{0.1cm}
    \caption{The table presents parameterizations of dense and MoE Transformers, showing parameter scaling in big-$\Theta$ notation, $d_\text{input}$ is the dimensionality of the weight input. Dense transformer $\mu$P is indicated in \textcolor{blue}{blue}. MoE parameterizations build on dense $\mu$P. \textit{simple}P MoE is heuristics marked in \textcolor{red}{red}, while the full-fledged theoretically grounded $\mu$P MoE is shown in \textcolor{darkgreen}{green}.}
    \label{tab:simple_parameterization}
    \vspace{0.1cm}
\end{table*}

\section{Principled approach for $\mu$P for MoE}\label{sec:mup_for_moe}

In this section, we analyze the training dynamics of Mixture-of-Experts (MoE) \citep{fedus2022switch} and derive a parameterization that ensures feature learning across all weights in the infinite width limit.

\subsection{Definitions and notation}

We scale the model width, with the number of experts and top-$k$ kept fixed. The hidden dimension of each expert grows proportionally with the model width, so that the ratio between them remains constant. We model the MoE layer as a Switch Transformer \citep{fedus2022switch}, which consists of:

\begin{itemize}
\vspace{-\topsep}
    \item A router matrix $R_0 \in \mathbb{R}^{n_\text{experts} \times n}$,
    \item Two weights of an MLP expert $E$: 
    $
        E_1 \in \mathbb{R}^{n_\text{experts} \times 4n \times n}, \quad
        E_2 \in \mathbb{R}^{n_\text{experts} \times n \times 4n}.
    $
\vspace{-\topsep}
\end{itemize}

Here $n$ denotes the model width, which gets scaled to infinity.

The forward pass of the MoE layer is computed as:
\begin{align*}
    \mathrm{MoE}(x) &= E(x)^T R(x).\\
    E(x) &= E_2  \mathrm{ReLU}(E_1 x),\\
    R(x) &= \mathrm{top\text{-}k}(\mathrm{softmax}(R_0x)).
\end{align*}

We adopt the notation of \citet{yang2022tuning}. A vector $v \in \R^n$ is said to be $\Theta(n^a)$ if 
\[
\frac{\|v\|^2}{n} = \Theta(n^{2a}),
\]
where $\|\cdot\|$ is the Euclidean norm. Intuitively, this means that a typical entry of $v$ has magnitude $\Theta(n^a)$. Analogous definitions apply for $O(n^a)$ and $\Omega(n^a)$, and extend naturally to matrices.

All layers outside of MoE blocks are assumed to follow the TP5 parameterization.

\subsection{Intuition}

TP5 categorizes weight matrices into three types based on their input and output dimensions. Sufficient conditions are:
\begin{itemize}\label{def:weight_conditions}
    \item \textbf{Input weight:} maps fixed-size $\to$ unbounded; initialized as $\Theta(1)$; gradient updates of size $\Theta(1)$.
    \item \textbf{Hidden weight:} maps unbounded $\to$ unbounded; initialized as $\Theta(1)$; gradient updates of size $\Theta(1/n)$.
    \item \textbf{Output weight:} maps unbounded $\to$ fixed-size; initialized as $\Theta(1)$; gradient updates of size $\Theta(1)$.
\end{itemize}

By inspection, the expert weights $(E_1,E_2)$ map unbounded to unbounded dimensions, suggesting they should behave as \textit{hidden weights}. The router weight $R_0$ maps from unbounded to fixed-size, making it a candidate \textit{output weight}. We verify this classification in Section~\ref{sub:derivation} by analyzing initialization and gradient scaling.

\subsection{Desiderata}\label{sub:desiderata}

Following TP5, a $\mu$-parameterized model must satisfy:
\vspace{0.2cm}
\begin{enumerate}
\setlength{\itemsep}{3pt}
\setlength{\parskip}{0pt}
\setlength{\parsep}{0pt}
    \item At initialization, all hidden representations $h(x)$ in the network should scale as $\Theta(1)$. 
    \item The model output logits $f(x)$ should be $O(1)$ at initialization.
    \item After one optimization step, the changes to hidden representations $\Delta h(x)$ and output logits $\Delta f(x)$ should be $\Theta(1)$.
\end{enumerate}

These conditions guarantee stable layer weights and outputs across model widths, ensuring feature learning \citep{yang2022tuning}.

\begin{figure*}
    \centering
    \includegraphics[width=\linewidth, trim=3.5cm 0 3.5cm 1cm, clip]{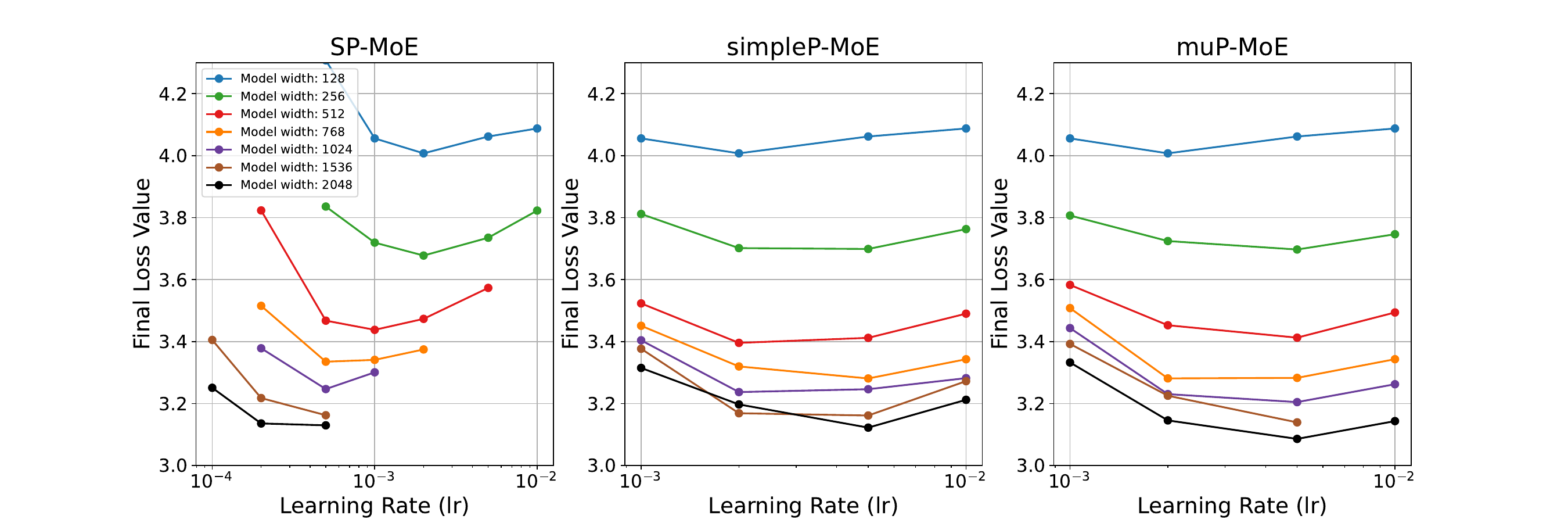}
    \caption{Learning rate transfer in MoE. \textbf{Left:} Standard Parameterization (SP). \textbf{Middle:} Our \textit{simple}P, treating each expert as a feed-forward layer. \textbf{Right:} Our $\mu$P for MoE with both router and expert reparametrization. Under SP, the optimal learning rate depends strongly on width, while both reparameterizations achieve transfer across widths.}
     \label{fig:lr_plot_moe}
\end{figure*}

\subsection{Derivation}\label{sub:derivation}
In this section, we use the parameterization from Table~\ref{tab:tp5_parameterization}. It is analogous to the parameterization from Table~\ref{tab:simple_parameterization}, on which we performed our experiments (See \citet{yang2022tuning} Appendix B).

\begin{table*}
    \centering
    \begin{tabular}{l|cc}
    \toprule
    \textbf{Weight type} & \textbf{Init. Var.} & \textbf{LR (Adam)} \\
    \midrule
    Input weight & \textcolor{gray}{$1/d_\text{input}$} & \textcolor{gray}{$1.0$} \\
    Output weight & \textcolor{blue}{$1/d_\text{input}^2$} & \textcolor{blue}{$1/d_\text{input}$} \\
    Hidden weight & \textcolor{gray}{$1/d_\text{input}$} & \textcolor{blue}{$1/d_\text{input}$}\\
    \bottomrule
    \end{tabular}
    \vspace{0.1cm}
    \caption{Parameterization from TP5 used in our derivation. Changes introduced by $\mu$P are highlighted in \textcolor{blue}{blue}; quantities unchanged from SP are shown in \textcolor{gray}{gray}.}
    \label{tab:tp5_parameterization}
    \vspace{0.1cm}
\end{table*}

\subsubsection{Initialization}
If we initialize $E_1, E_2$ as in a dense MLP, and set the router using the same variance and scaling factor as the output layer in TP5, we obtain:

If $E_1,E_2$ are initialized like \textit{hidden weights}, and $R_0$ with the variance of an \textit{output weight} (See Table~\ref{tab:tp5_parameterization}), thus we obtain:

\vspace{0.2cm}
\begin{itemize}
  \setlength\itemsep{1pt}
  \setlength\parskip{1pt}
    \setlength\parsep{0pt}
  \setlength\topsep{0pt}
  \setlength\partopsep{0pt}
  \item $E(x) = \Theta(1)$
  \item $R_0x = O(1)$
  \item $R(x) = \Theta(1)$
  \item $\mathrm{MoE}(x) = \Theta(1)$
\end{itemize}
\vspace{0.2cm}

The term $R_0x$ is non-standard, but it does not affect the overall output scale, which remains $\Theta(1)$.

\subsubsection{Optimizer step}


Assuming all other modules satisfy the desiderata in Section~\ref{sub:desiderata}, gradients with respect to hidden-layer outputs are $\Theta(1/n)$. For MoE, this implies:
\[
\nabla \mathrm{MoE}(x) = \Theta(1/n).
\]

Now we can calculate $\nabla E(x)$:
\begin{equation}
\begin{aligned}
\nabla E(x) &= R(x)  \nabla \mathrm{MoE}(x)^T \\
&=\Theta(1) \cdot \Theta(1/n) = \Theta(1/n)
\end{aligned}
\end{equation}

Here, the summation is over a finite number of experts, 
so it does not change the asymptotic scaling, i.e., the $\Theta(\cdot)$ order remains the same. 
Since $R(x)$ and $\mathrm{MoE}(x)$ are 1-dimensional vectors, 
each entry of $\nabla E(x)$ is simply the product of the corresponding entry sizes.

Similarly calculate $\nabla R(x)$:

\begin{equation}
\begin{aligned}
\nabla R(x) &= E(x)  \nabla \mathrm{MoE}(x) \\
            &= \Theta(1) \cdot \Theta(1/n) \cdot n = \Theta(1)
\end{aligned}
\end{equation}
\vspace{0.3cm}


The multiplication by $n$ in this equality follows from $E(x)$ and $\nabla \mathrm{MoE}(x)$ being correlated. We prove that correlation in Appendix~\ref{sec:eg-cov}. In short, for correlated vectors $v, u \in \R^n$ quantity $v^T u$ has expected size $\Theta(v)\Theta(u)\cdot \text{corr}(v, u) \cdot n$, which follows from the Law of Large Numbers.

$\nabla E_1x, \nabla E_2\text{ReLU}(E_1x)$ are $\Theta(1/n)$ since they mimic standard MLP layers.
The Softmax and top-k over constant k do not change the size of the gradient
For $R(x)$, since $R_0x$ is $O(1)$, the gradient over $R_0$ is still $\Theta(1)$.
That means $E_1, E_2, R_0$ receive the same gradient sizes as hidden weights and output weight, respectively, so they behave in the same way in training as their weight types from TP5 

In conclusion, $E_1, E_2$ should be parameterized as \textit{hidden weights}, while $R_0$ should be treated as an \textit{output weight}.

\section{Experimental Results and Alternative Views on Hyperparameter Transfer in MoE}
\label{s:exp}

This section presents experimental results on learning rate transfer in MoE Transformers, providing empirical validation of our parameterization. Additionally, we explore the transferability of learning rate when scaling different MoE dimensions.

\subsection{Alternative parameterization for MoE models}
In Section~\ref{sec:mup_for_moe}, we develop a theory for $\mu$P for MoE where we re-parametrize both router and experts. In addition to this theoretically grounded scheme, we also evaluate a simplified parameterization, which we call \textit{simple}-Parameterization (\textit{simple}P). \textit{simple}P applies the $\mu$P rules of dense Transformers directly to each expert, leveraging the structural similarity between experts and Transformer MLP blocks. The router, however, is left unmodified. Table~\ref{tab:simple_parameterization} summarizes both parameterizations. Importantly, \textit{simple}P is \emph{not} a $\mu$-Parameterization, as it lacks theoretical guarantees for feature learning.

\subsection{Model width}
We evaluate the transferability of the learning rate across model widths (Figure~\ref{fig:lr_plot_moe} and Figure~\ref{fig:lr_plot_moe_small} in Appendix). For both $\mu$P and \textit{simple}P, the optimal learning rate shifts slightly upward as width increases. A similar trend appears in Tensor Programs~5 and in our dense model experiments (Figure~\ref{fig:mup_dense}). We hypothesize that this shift arises from instabilities in architectures with large depth-to-width ratios, which may underperform on high learning rates. A more detailed analysis is left for future work.
Overall, these results support the transferability of learning rates under both $\mu$P and \textit{simple}P.

\subsection{Computational savings in tuning}
To illustrate the efficiency gains from $\mu$Transfer, consider training a model with width $2048$. A full run requires roughly $4 \times 10^{18}$ FLOPs. In contrast, tuning on a smaller model of width $128$ requires only $1.5 \times 10^{16}$ FLOPs, over $250$ times cheaper. Thus, tuning costs become negligible compared to the final training run, underscoring the practical value of $\mu$Transfer.

\subsection{Scaling other MoE dimensions}
In the previous section, we showed MoE parameterizations for varying model widths. In this section, we investigate whether scaling the MoE architecture in two other dimensions necessitates reparameterization.

\textbf{Number of experts.} Increasing the number of experts expands model capacity without increasing computational cost \citep{clark_unified, ludziejewski2025jointmoescalinglaws}. As shown in Figure~\ref{fig:other_moe_aspects}(a), models with more experts achieve lower final loss, while the optimal learning rate shifts slightly downward. The shift is small and comparable to that observed under width scaling, though it may indicate a more fundamental effect. At the scale of our experiments, however, verifying a dedicated parameterization for expert scaling is infeasible, as the results under SP remain within the margin of error. Overall, we find that changing the number of experts does not alter the stability of the under $muP$ for MoE parameter transfer.

\textbf{Granularity.} Granularity, described in Section~\ref{s:rel}, adjusts expert size while keeping computational cost fixed. Increasing granularity within a reasonable limit improves the model performance \citep{ludziejewskiscaling}. It can be used in practice to match hardware constraints \citep{dai2024deepseekmoe}. Figure~\ref{fig:other_moe_aspects}(b) shows that granularity changes break learning rate transfer. We attribute this to changes to scaling $\text{top-}k$ or the reduced hidden dimension of each expert. Since both settings alter router output dimensions, they lie outside the assumptions of our theory. We leave a full theoretical treatment of these results for future work.

\begin{figure}[htbp]
    \centering
    \begin{subfigure}[b]{0.48\linewidth}
        \centering
        \includegraphics[width=\linewidth]{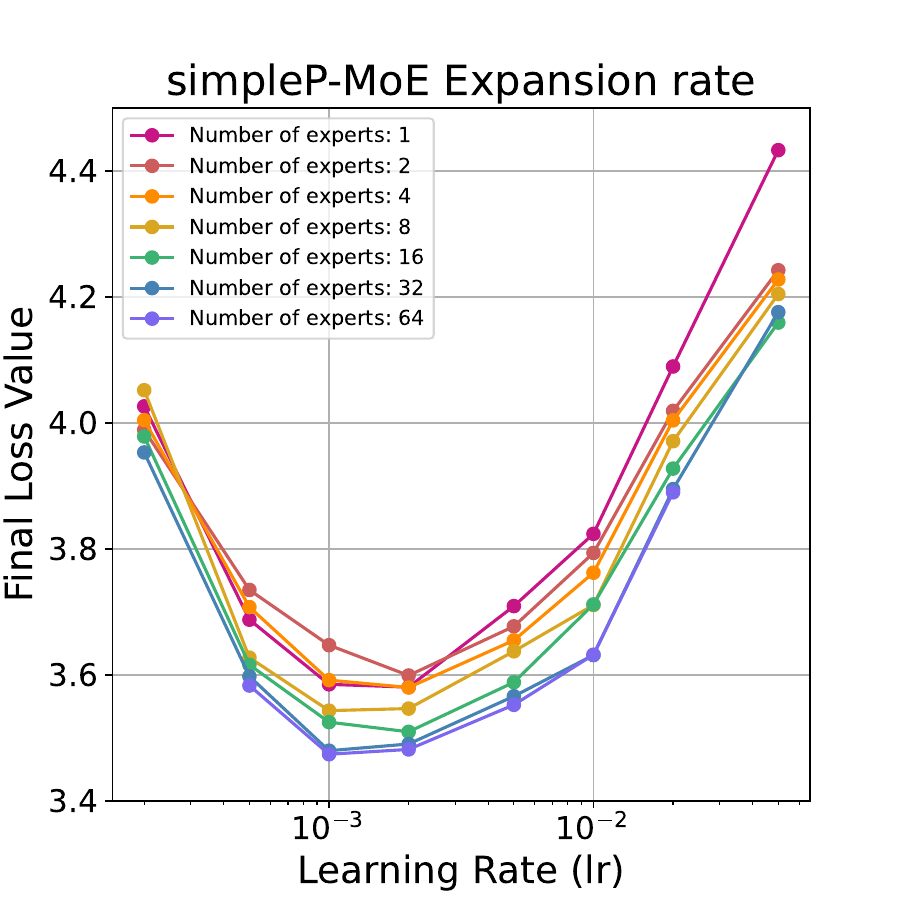}
        \label{fig:image1}
    \end{subfigure}
    \hfill
    \begin{subfigure}[b]{0.48\linewidth}
        \centering
        \includegraphics[width=\linewidth]{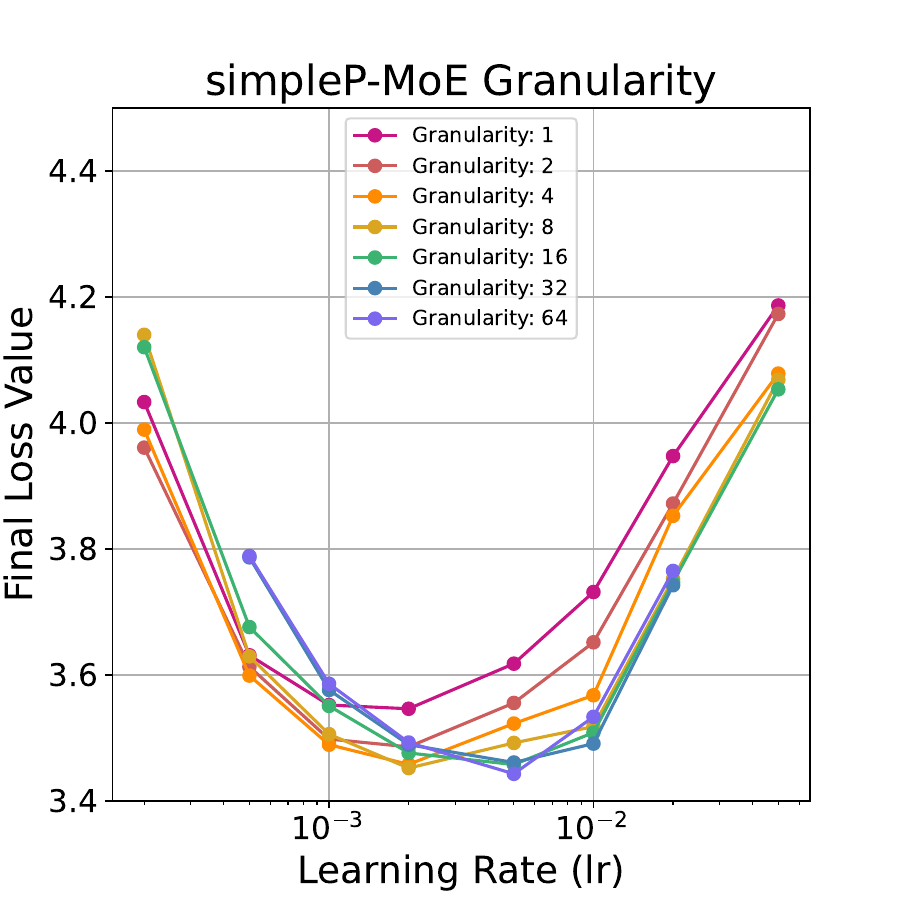}
        \label{fig:image2}
    \end{subfigure}
    \vspace{0.4cm}
    \caption{(a) Varying the number of experts. Given our $\mu$P for MoE, the optimal learning rate is preserved across a varied number of experts. (b) Varying granularity. Learning rate is not preserved across different granularities.}
    \label{fig:other_moe_aspects}
    \vspace{0.3cm}
\end{figure}

\section{Conclusions}
We introduced a $\mu$-Parameterization for Mixture-of-Experts (MoE) that preserves feature learning across widths by classifying expert weights as hidden and the router as output within the TP5 framework. Our theory implies stable gradient scales and width-invariant training dynamics, and our experiments confirm that the optimal learning rate transfers reliably across model sizes under our $\mu$P. Moreover, a simplified expert-only parameterization \textit{simple}P also enables learning rate transfer in the scope of our experiments. 

Beyond width, we examined other MoE scaling axes. Increasing the number of experts improves final loss with only a minor shift in the optimal learning rate, comparable to width scaling. In contrast, changing granularity breaks learning-rate transfer, likely due to altered router behavior and expert hidden sizes—conditions outside our current assumptions. These findings delineate where zero-shot hyperparameter transfer is robust in MoE, and where reparameterization or new theory is needed the most.

\section{Acknowledgments}

We gratefully acknowledge Polish high-performance computing infrastructure PLGrid (HPC Center: ACK Cyfronet AGH) for providing computer facilities and support within computational grants no. PLG/2024/017845 and no. PLG/2024/017060. Part of the experiments utilized computational resources provided by \href{https://writer.com/}{Writer}.

\bibliography{iclr2026_conference}

\begin{thebibliography}{24}
\providecommand{\natexlab}[1]{#1}
\providecommand{\url}[1]{\texttt{#1}}
\expandafter\ifx\csname urlstyle\endcsname\relax
  \providecommand{\doi}[1]{doi: #1}\else
  \providecommand{\doi}{doi: \begingroup \urlstyle{rm}\Url}\fi

\bibitem[Clark et~al.(2022)Clark, De~Las~Casas, Guy, Mensch, Paganini, Hoffmann, Damoc, Hechtman, Cai, Borgeaud, Van Den~Driessche, Rutherford, Hennigan, Johnson, Cassirer, Jones, Buchatskaya, Budden, Sifre, Osindero, Vinyals, Ranzato, Rae, Elsen, Kavukcuoglu, and Simonyan]{clark_unified}
Aidan Clark, Diego De~Las~Casas, Aurelia Guy, Arthur Mensch, Michela Paganini, Jordan Hoffmann, Bogdan Damoc, Blake Hechtman, Trevor Cai, Sebastian Borgeaud, George~Bm Van Den~Driessche, Eliza Rutherford, Tom Hennigan, Matthew~J Johnson, Albin Cassirer, Chris Jones, Elena Buchatskaya, David Budden, Laurent Sifre, Simon Osindero, Oriol Vinyals, Marc'Aurelio Ranzato, Jack Rae, Erich Elsen, Koray Kavukcuoglu, and Karen Simonyan.
\newblock Unified scaling laws for routed language models.
\newblock In Kamalika Chaudhuri, Stefanie Jegelka, Le~Song, Csaba Szepesvari, Gang Niu, and Sivan Sabato (eds.), \emph{Proceedings of the 39th International Conference on Machine Learning}, volume 162 of \emph{Proceedings of Machine Learning Research}, pp.\  4057--4086. PMLR, 17--23 Jul 2022.
\newblock URL \url{https://proceedings.mlr.press/v162/clark22a.html}.

\bibitem[Dai et~al.(2024)Dai, Deng, Zhao, Xu, Gao, Chen, Li, Zeng, Yu, Wu, Xie, Li, Huang, Luo, Ruan, Sui, and Liang]{dai2024deepseekmoe}
Damai Dai, Chengqi Deng, Chenggang Zhao, R.~X. Xu, Huazuo Gao, Deli Chen, Jiashi Li, Wangding Zeng, Xingkai Yu, Y.~Wu, Zhenda Xie, Y.~K. Li, Panpan Huang, Fuli Luo, Chong Ruan, Zhifang Sui, and Wenfeng Liang.
\newblock Deepseekmoe: Towards ultimate expert specialization in mixture-of-experts language models, 2024.

\bibitem[DeepSeek-AI et~al.(2025)DeepSeek-AI, Liu, Feng, Xue, Wang, Wu, Lu, Zhao, Deng, Zhang, Ruan, Dai, Guo, Yang, Chen, Ji, Li, Lin, Dai, Luo, Hao, Chen, Li, Zhang, Bao, Xu, Wang, Zhang, Ding, Xin, Gao, Li, Qu, Cai, Liang, Guo, Ni, Li, Wang, Chen, Chen, Yuan, Qiu, Li, Song, Dong, Hu, Gao, Guan, Huang, Yu, Wang, Zhang, Xu, Xia, Zhao, Wang, Zhang, Li, Wang, Zhang, Zhang, Tang, Li, Tian, Huang, Wang, Zhang, Wang, Zhu, Chen, Du, Chen, Jin, Ge, Zhang, Pan, Wang, Xu, Zhang, Chen, Li, Lu, Zhou, Chen, Wu, Ye, Ye, Ma, Wang, Zhou, Yu, Zhou, Pan, Wang, Yun, Pei, Sun, Xiao, Zeng, Zhao, An, Liu, Liang, Gao, Yu, Zhang, Li, Jin, Wang, Bi, Liu, Wang, Shen, Chen, Zhang, Chen, Nie, Sun, Wang, Cheng, Liu, Xie, Liu, Yu, Song, Shan, Zhou, Yang, Li, Su, Lin, Li, Wang, Wei, Zhu, Zhang, Xu, Xu, Huang, Li, Zhao, Sun, Li, Wang, Yu, Zheng, Zhang, Shi, Xiong, He, Tang, Piao, Wang, Tan, Ma, Liu, Guo, Wu, Ou, Zhu, Wang, Gong, Zou, He, Zha, Xiong, Ma, Yan, Luo, You, Liu, Zhou, Wu, Ren, Ren, Sha, Fu, Xu, Huang, Zhang, Xie, Zhang, Hao,
  Gou, Ma, Yan, Shao, Xu, Wu, Zhang, Li, Gu, Zhu, Liu, Li, Xie, Song, Gao, and Pan]{deepseekai2025deepseekv3technicalreport}
DeepSeek-AI, Aixin Liu, Bei Feng, Bing Xue, Bingxuan Wang, Bochao Wu, Chengda Lu, Chenggang Zhao, Chengqi Deng, Chenyu Zhang, Chong Ruan, Damai Dai, Daya Guo, Dejian Yang, Deli Chen, Dongjie Ji, Erhang Li, Fangyun Lin, Fucong Dai, Fuli Luo, Guangbo Hao, Guanting Chen, Guowei Li, H.~Zhang, Han Bao, Hanwei Xu, Haocheng Wang, Haowei Zhang, Honghui Ding, Huajian Xin, Huazuo Gao, Hui Li, Hui Qu, J.~L. Cai, Jian Liang, Jianzhong Guo, Jiaqi Ni, Jiashi Li, Jiawei Wang, Jin Chen, Jingchang Chen, Jingyang Yuan, Junjie Qiu, Junlong Li, Junxiao Song, Kai Dong, Kai Hu, Kaige Gao, Kang Guan, Kexin Huang, Kuai Yu, Lean Wang, Lecong Zhang, Lei Xu, Leyi Xia, Liang Zhao, Litong Wang, Liyue Zhang, Meng Li, Miaojun Wang, Mingchuan Zhang, Minghua Zhang, Minghui Tang, Mingming Li, Ning Tian, Panpan Huang, Peiyi Wang, Peng Zhang, Qiancheng Wang, Qihao Zhu, Qinyu Chen, Qiushi Du, R.~J. Chen, R.~L. Jin, Ruiqi Ge, Ruisong Zhang, Ruizhe Pan, Runji Wang, Runxin Xu, Ruoyu Zhang, Ruyi Chen, S.~S. Li, Shanghao Lu, Shangyan Zhou, Shanhuang
  Chen, Shaoqing Wu, Shengfeng Ye, Shengfeng Ye, Shirong Ma, Shiyu Wang, Shuang Zhou, Shuiping Yu, Shunfeng Zhou, Shuting Pan, T.~Wang, Tao Yun, Tian Pei, Tianyu Sun, W.~L. Xiao, Wangding Zeng, Wanjia Zhao, Wei An, Wen Liu, Wenfeng Liang, Wenjun Gao, Wenqin Yu, Wentao Zhang, X.~Q. Li, Xiangyue Jin, Xianzu Wang, Xiao Bi, Xiaodong Liu, Xiaohan Wang, Xiaojin Shen, Xiaokang Chen, Xiaokang Zhang, Xiaosha Chen, Xiaotao Nie, Xiaowen Sun, Xiaoxiang Wang, Xin Cheng, Xin Liu, Xin Xie, Xingchao Liu, Xingkai Yu, Xinnan Song, Xinxia Shan, Xinyi Zhou, Xinyu Yang, Xinyuan Li, Xuecheng Su, Xuheng Lin, Y.~K. Li, Y.~Q. Wang, Y.~X. Wei, Y.~X. Zhu, Yang Zhang, Yanhong Xu, Yanhong Xu, Yanping Huang, Yao Li, Yao Zhao, Yaofeng Sun, Yaohui Li, Yaohui Wang, Yi~Yu, Yi~Zheng, Yichao Zhang, Yifan Shi, Yiliang Xiong, Ying He, Ying Tang, Yishi Piao, Yisong Wang, Yixuan Tan, Yiyang Ma, Yiyuan Liu, Yongqiang Guo, Yu~Wu, Yuan Ou, Yuchen Zhu, Yuduan Wang, Yue Gong, Yuheng Zou, Yujia He, Yukun Zha, Yunfan Xiong, Yunxian Ma, Yuting Yan, Yuxiang
  Luo, Yuxiang You, Yuxuan Liu, Yuyang Zhou, Z.~F. Wu, Z.~Z. Ren, Zehui Ren, Zhangli Sha, Zhe Fu, Zhean Xu, Zhen Huang, Zhen Zhang, Zhenda Xie, Zhengyan Zhang, Zhewen Hao, Zhibin Gou, Zhicheng Ma, Zhigang Yan, Zhihong Shao, Zhipeng Xu, Zhiyu Wu, Zhongyu Zhang, Zhuoshu Li, Zihui Gu, Zijia Zhu, Zijun Liu, Zilin Li, Ziwei Xie, Ziyang Song, Ziyi Gao, and Zizheng Pan.
\newblock Deepseek-v3 technical report, 2025.
\newblock URL \url{https://arxiv.org/abs/2412.19437}.

\bibitem[Dey et~al.(2025)Dey, Zhang, Noci, Li, Bordelon, Bergsma, Pehlevan, Hanin, and Hestness]{completeP}
Nolan Dey, Bin~Claire Zhang, Lorenzo Noci, Mufan Li, Blake Bordelon, Shane Bergsma, Cengiz Pehlevan, Boris Hanin, and Joel Hestness.
\newblock Don't be lazy: Completep enables compute-efficient deep transformers, 2025.
\newblock URL \url{https://arxiv.org/abs/2505.01618}.

\bibitem[Du et~al.(2022)Du, Huang, Dai, Tong, Lepikhin, Xu, Krikun, Zhou, Yu, Firat, Zoph, Fedus, Bosma, Zhou, Wang, Wang, Webster, Pellat, Robinson, Meier-Hellstern, Duke, Dixon, Zhang, Le, Wu, Chen, and Cui]{du2022glam}
Nan Du, Yanping Huang, Andrew~M. Dai, Simon Tong, Dmitry Lepikhin, Yuanzhong Xu, Maxim Krikun, Yanqi Zhou, Adams~Wei Yu, Orhan Firat, Barret Zoph, Liam Fedus, Maarten Bosma, Zongwei Zhou, Tao Wang, Yu~Emma Wang, Kellie Webster, Marie Pellat, Kevin Robinson, Kathleen Meier-Hellstern, Toju Duke, Lucas Dixon, Kun Zhang, Quoc~V Le, Yonghui Wu, Zhifeng Chen, and Claire Cui.
\newblock Glam: Efficient scaling of language models with mixture-of-experts, 2022.

\bibitem[Everett et~al.(2024)Everett, Xiao, Wortsman, Alemi, Novak, Liu, Gur, Sohl-Dickstein, Kaelbling, Lee, and Pennington]{scaling_exponents}
Katie Everett, Lechao Xiao, Mitchell Wortsman, Alexander~A. Alemi, Roman Novak, Peter~J. Liu, Izzeddin Gur, Jascha Sohl-Dickstein, Leslie~Pack Kaelbling, Jaehoon Lee, and Jeffrey Pennington.
\newblock Scaling exponents across parameterizations and optimizers, 2024.
\newblock URL \url{https://arxiv.org/abs/2407.05872}.

\bibitem[Fedus et~al.(2022)Fedus, Zoph, and Shazeer]{fedus2022switch}
William Fedus, Barret Zoph, and Noam Shazeer.
\newblock Switch transformers: Scaling to trillion parameter models with simple and efficient sparsity.
\newblock \emph{Journal of Machine Learning Research}, 23\penalty0 (120):\penalty0 1--39, 2022.

\bibitem[Jacobs et~al.(1991)Jacobs, Jordan, Nowlan, and Hinton]{moe1991}
Robert~A. Jacobs, Michael~I. Jordan, Steven~J. Nowlan, and Geoffrey~E. Hinton.
\newblock Adaptive mixtures of local experts.
\newblock \emph{Neural Computation}, 3\penalty0 (1):\penalty0 79--87, 1991.
\newblock \doi{10.1162/neco.1991.3.1.79}.

\bibitem[Lepikhin et~al.(2020)Lepikhin, Lee, Xu, Chen, Firat, Huang, Krikun, Shazeer, Sepassi, Tucker, and Zhou]{lepikhin2020gshard}
Denis Lepikhin, HyoukJoong Lee, Yuanzhong Xu, Dehao Chen, Orhan Firat, Yanping Huang, Maxim Krikun, Noam Shazeer, Ramin Sepassi, Paul Tucker, and Colin Zhou.
\newblock Gshard: Scaling giant models with conditional computation and automatic sharding.
\newblock In \emph{Proceedings of the 37th International Conference on Machine Learning (ICML)}, 2020.
\newblock URL \url{https://arxiv.org/abs/2006.16668}.

\bibitem[Loshchilov \& Hutter(2019)Loshchilov and Hutter]{loshchilov2019decoupled}
Ilya Loshchilov and Frank Hutter.
\newblock Decoupled weight decay regularization, 2019.

\bibitem[Ludziejewski et~al.(2024)Ludziejewski, Krajewski, Adamczewski, Pi\'{o}ro, Krutul, Antoniak, Ciebiera, Kr\'{o}l, Odrzyg\'{o}\'{z}d\'{z}, Sankowski, Cygan, and Jaszczur]{ludziejewskiscaling}
Jan Ludziejewski, Jakub Krajewski, Kamil Adamczewski, Maciej Pi\'{o}ro, Micha{\l} Krutul, Szymon Antoniak, Kamil Ciebiera, Krystian Kr\'{o}l, Tomasz Odrzyg\'{o}\'{z}d\'{z}, Piotr Sankowski, Marek Cygan, and Sebastian Jaszczur.
\newblock Scaling laws for fine-grained mixture of experts.
\newblock In Ruslan Salakhutdinov, Zico Kolter, Katherine Heller, Adrian Weller, Nuria Oliver, Jonathan Scarlett, and Felix Berkenkamp (eds.), \emph{Proceedings of the 41st International Conference on Machine Learning}, volume 235 of \emph{Proceedings of Machine Learning Research}, pp.\  33270--33288. PMLR, 21--27 Jul 2024.
\newblock URL \url{https://proceedings.mlr.press/v235/ludziejewski24a.html}.

\bibitem[Ludziejewski et~al.(2025)Ludziejewski, Pióro, Krajewski, Stefaniak, Krutul, Małaśnicki, Cygan, Sankowski, Adamczewski, Miłoś, and Jaszczur]{ludziejewski2025jointmoescalinglaws}
Jan Ludziejewski, Maciej Pióro, Jakub Krajewski, Maciej Stefaniak, Michał Krutul, Jan Małaśnicki, Marek Cygan, Piotr Sankowski, Kamil Adamczewski, Piotr Miłoś, and Sebastian Jaszczur.
\newblock Joint moe scaling laws: Mixture of experts can be memory efficient, 2025.
\newblock URL \url{https://arxiv.org/abs/2502.05172}.

\bibitem[OpenAI et~al.(2025)OpenAI, :, Agarwal, Ahmad, Ai, Altman, Applebaum, Arbus, Arora, Bai, Baker, Bao, Barak, Bennett, Bertao, Brett, Brevdo, Brockman, Bubeck, Chang, Chen, Chen, Cheung, Clark, Cook, Dukhan, Dvorak, Fives, Fomenko, Garipov, Georgiev, Glaese, Gogineni, Goucher, Gross, Guzman, Hallman, Hehir, Heidecke, Helyar, Hu, Huet, Huh, Jain, Johnson, Koch, Kofman, Kundel, Kwon, Kyrylov, Le, Leclerc, Lennon, Lessans, Lezcano-Casado, Li, Li, Lin, Liss, Lily, Liu, Liu, Lu, Lu, Martinovic, McCallum, McGrath, McKinney, McLaughlin, Mei, Mostovoy, Mu, Myles, Neitz, Nichol, Pachocki, Paino, Palmie, Pantuliano, Parascandolo, Park, Pathak, Paz, Peran, Pimenov, Pokrass, Proehl, Qiu, Raila, Raso, Ren, Richardson, Robinson, Rotsted, Salman, Sanjeev, Schwarzer, Sculley, Sikchi, Simon, Singhal, Song, Stuckey, Sun, Tillet, Toizer, Tsimpourlas, Vyas, Wallace, Wang, Wang, Watkins, Weil, Wendling, Whinnery, Whitney, Wong, Yang, Yang, Yasunaga, Ying, Zaremba, Zhan, Zhang, Zhang, Zhang, and
  Zhao]{openai2025gptoss120bgptoss20bmodel}
OpenAI, :, Sandhini Agarwal, Lama Ahmad, Jason Ai, Sam Altman, Andy Applebaum, Edwin Arbus, Rahul~K. Arora, Yu~Bai, Bowen Baker, Haiming Bao, Boaz Barak, Ally Bennett, Tyler Bertao, Nivedita Brett, Eugene Brevdo, Greg Brockman, Sebastien Bubeck, Che Chang, Kai Chen, Mark Chen, Enoch Cheung, Aidan Clark, Dan Cook, Marat Dukhan, Casey Dvorak, Kevin Fives, Vlad Fomenko, Timur Garipov, Kristian Georgiev, Mia Glaese, Tarun Gogineni, Adam Goucher, Lukas Gross, Katia~Gil Guzman, John Hallman, Jackie Hehir, Johannes Heidecke, Alec Helyar, Haitang Hu, Romain Huet, Jacob Huh, Saachi Jain, Zach Johnson, Chris Koch, Irina Kofman, Dominik Kundel, Jason Kwon, Volodymyr Kyrylov, Elaine~Ya Le, Guillaume Leclerc, James~Park Lennon, Scott Lessans, Mario Lezcano-Casado, Yuanzhi Li, Zhuohan Li, Ji~Lin, Jordan Liss, Lily, Liu, Jiancheng Liu, Kevin Lu, Chris Lu, Zoran Martinovic, Lindsay McCallum, Josh McGrath, Scott McKinney, Aidan McLaughlin, Song Mei, Steve Mostovoy, Tong Mu, Gideon Myles, Alexander Neitz, Alex Nichol, Jakub
  Pachocki, Alex Paino, Dana Palmie, Ashley Pantuliano, Giambattista Parascandolo, Jongsoo Park, Leher Pathak, Carolina Paz, Ludovic Peran, Dmitry Pimenov, Michelle Pokrass, Elizabeth Proehl, Huida Qiu, Gaby Raila, Filippo Raso, Hongyu Ren, Kimmy Richardson, David Robinson, Bob Rotsted, Hadi Salman, Suvansh Sanjeev, Max Schwarzer, D.~Sculley, Harshit Sikchi, Kendal Simon, Karan Singhal, Yang Song, Dane Stuckey, Zhiqing Sun, Philippe Tillet, Sam Toizer, Foivos Tsimpourlas, Nikhil Vyas, Eric Wallace, Xin Wang, Miles Wang, Olivia Watkins, Kevin Weil, Amy Wendling, Kevin Whinnery, Cedric Whitney, Hannah Wong, Lin Yang, Yu~Yang, Michihiro Yasunaga, Kristen Ying, Wojciech Zaremba, Wenting Zhan, Cyril Zhang, Brian Zhang, Eddie Zhang, and Shengjia Zhao.
\newblock gpt-oss-120b \& gpt-oss-20b model card, 2025.
\newblock URL \url{https://arxiv.org/abs/2508.10925}.

\bibitem[Radford et~al.(2018)Radford, Narasimhan, Salimans, and Sutskever]{radford2018improving}
Alec Radford, Karthik Narasimhan, Tim Salimans, and Ilya Sutskever.
\newblock Improving language understanding by generative pre-training.
\newblock 2018.

\bibitem[Raffel et~al.(2020)Raffel, Shazeer, Roberts, Lee, Narang, Matena, Zhou, Li, and Liu]{raffel2020exploring}
Colin Raffel, Noam Shazeer, Adam Roberts, Katherine Lee, Sharan Narang, Michael Matena, Yanqi Zhou, Wei Li, and Peter~J. Liu.
\newblock Exploring the limits of transfer learning with a unified text-to-text transformer.
\newblock In \emph{Proceedings of the 37th International Conference on Machine Learning}, pp.\  5541--5551. PMLR, 2020.
\newblock URL \url{https://arxiv.org/abs/1910.10683}.

\bibitem[Shazeer et~al.(2017)Shazeer, Mirhoseini, Maziarz, Davis, Le, Hinton, and Dean]{shazeer2017outrageously}
Noam Shazeer, Azalia Mirhoseini, Krzysztof Maziarz, Andy Davis, Quoc Le, Geoffrey Hinton, and Jeff Dean.
\newblock Outrageously large neural networks: The sparsely-gated mixture-of-experts layer, 2017.

\bibitem[Team et~al.(2025)Team, Bai, Bao, Chen, Chen, Chen, Chen, Chen, Chen, Chen, Chen, Cui, Ding, Dong, Du, Du, Du, Du, Fan, Feng, Fu, Gao, Gao, Gao, Gao, Gu, Guan, Guo, Guo, Hu, Hao, He, He, He, Hong, Hu, Hu, Huang, Huang, Huang, Jiang, Jiang, Jin, Kang, Lai, Li, Li, Li, Li, Li, Li, Li, Li, Li, Lin, Lin, Lin, Liu, Liu, Liu, Liu, Liu, Liu, Liu, Liu, Liu, Liu, Liu, Liu, Liu, Liu, Liu, Lu, Lu, Ma, Ma, Ma, Mao, Mei, Men, Miao, Pan, Peng, Qin, Qu, Shang, Shi, Shi, Song, Su, Su, Sun, Sung, Tang, Tao, Teng, Wang, Wang, Wang, Wang, Wang, Wang, Wang, Wang, Wang, Wang, Wang, Wang, Wang, Wang, Wang, Wang, Wang, Wei, Wei, Wu, Wu, Wu, Xiao, Xie, Xiong, Xu, Xu, Xu, Xu, Xu, Xu, Xu, Xu, Xu, Xu, Yan, Yan, Yang, Yang, Yang, Yang, Yang, Yao, Yao, Ye, Ye, Yin, Yu, Yuan, Yuan, Yuan, Zhan, Zhang, Zhang, Zhang, Zhang, Zhang, Zhang, Zhang, Zhang, Zhang, Zhang, Zhang, Zhao, Zhao, Zheng, Zheng, Zhou, Zhou, Zhou, Zhu, Zhuang, and Zu]{kimiteam2025kimik2openagentic}
Kimi Team, Yifan Bai, Yiping Bao, Guanduo Chen, Jiahao Chen, Ningxin Chen, Ruijue Chen, Yanru Chen, Yuankun Chen, Yutian Chen, Zhuofu Chen, Jialei Cui, Hao Ding, Mengnan Dong, Angang Du, Chenzhuang Du, Dikang Du, Yulun Du, Yu~Fan, Yichen Feng, Kelin Fu, Bofei Gao, Hongcheng Gao, Peizhong Gao, Tong Gao, Xinran Gu, Longyu Guan, Haiqing Guo, Jianhang Guo, Hao Hu, Xiaoru Hao, Tianhong He, Weiran He, Wenyang He, Chao Hong, Yangyang Hu, Zhenxing Hu, Weixiao Huang, Zhiqi Huang, Zihao Huang, Tao Jiang, Zhejun Jiang, Xinyi Jin, Yongsheng Kang, Guokun Lai, Cheng Li, Fang Li, Haoyang Li, Ming Li, Wentao Li, Yanhao Li, Yiwei Li, Zhaowei Li, Zheming Li, Hongzhan Lin, Xiaohan Lin, Zongyu Lin, Chengyin Liu, Chenyu Liu, Hongzhang Liu, Jingyuan Liu, Junqi Liu, Liang Liu, Shaowei Liu, T.~Y. Liu, Tianwei Liu, Weizhou Liu, Yangyang Liu, Yibo Liu, Yiping Liu, Yue Liu, Zhengying Liu, Enzhe Lu, Lijun Lu, Shengling Ma, Xinyu Ma, Yingwei Ma, Shaoguang Mao, Jie Mei, Xin Men, Yibo Miao, Siyuan Pan, Yebo Peng, Ruoyu Qin, Bowen Qu, Zeyu
  Shang, Lidong Shi, Shengyuan Shi, Feifan Song, Jianlin Su, Zhengyuan Su, Xinjie Sun, Flood Sung, Heyi Tang, Jiawen Tao, Qifeng Teng, Chensi Wang, Dinglu Wang, Feng Wang, Haiming Wang, Jianzhou Wang, Jiaxing Wang, Jinhong Wang, Shengjie Wang, Shuyi Wang, Yao Wang, Yejie Wang, Yiqin Wang, Yuxin Wang, Yuzhi Wang, Zhaoji Wang, Zhengtao Wang, Zhexu Wang, Chu Wei, Qianqian Wei, Wenhao Wu, Xingzhe Wu, Yuxin Wu, Chenjun Xiao, Xiaotong Xie, Weimin Xiong, Boyu Xu, Jing Xu, Jinjing Xu, L.~H. Xu, Lin Xu, Suting Xu, Weixin Xu, Xinran Xu, Yangchuan Xu, Ziyao Xu, Junjie Yan, Yuzi Yan, Xiaofei Yang, Ying Yang, Zhen Yang, Zhilin Yang, Zonghan Yang, Haotian Yao, Xingcheng Yao, Wenjie Ye, Zhuorui Ye, Bohong Yin, Longhui Yu, Enming Yuan, Hongbang Yuan, Mengjie Yuan, Haobing Zhan, Dehao Zhang, Hao Zhang, Wanlu Zhang, Xiaobin Zhang, Yangkun Zhang, Yizhi Zhang, Yongting Zhang, Yu~Zhang, Yutao Zhang, Yutong Zhang, Zheng Zhang, Haotian Zhao, Yikai Zhao, Huabin Zheng, Shaojie Zheng, Jianren Zhou, Xinyu Zhou, Zaida Zhou, Zhen Zhu,
  Weiyu Zhuang, and Xinxing Zu.
\newblock Kimi k2: Open agentic intelligence, 2025.
\newblock URL \url{https://arxiv.org/abs/2507.20534}.

\bibitem[Vaswani et~al.(2017)Vaswani, Shazeer, Parmar, Uszkoreit, Jones, Gomez, Kaiser, and Polosukhin]{vaswani2017attention}
Ashish Vaswani, Noam Shazeer, Niki Parmar, Jakob Uszkoreit, Llion Jones, Aidan~N Gomez, \L~ukasz Kaiser, and Illia Polosukhin.
\newblock Attention is all you need.
\newblock In I.~Guyon, U.~Von Luxburg, S.~Bengio, H.~Wallach, R.~Fergus, S.~Vishwanathan, and R.~Garnett (eds.), \emph{Advances in Neural Information Processing Systems}, volume~30. Curran Associates, Inc., 2017.
\newblock URL \url{https://proceedings.neurips.cc/paper_files/paper/2017/file/3f5ee243547dee91fbd053c1c4a845aa-Paper.pdf}.

\bibitem[Yang(2021)]{yang2021feature}
Greg Yang.
\newblock Tp4: Feature learning in infinite-width neural networks.
\newblock In \emph{International Conference on Learning Representations (ICLR)}, 2021.
\newblock URL \url{https://openreview.net/forum?id=K19h3z-4Z}.

\bibitem[Yang et~al.(2022)Yang, Hu, et~al.]{yang2022tuning}
Greg Yang, Edward~J Hu, et~al.
\newblock Tp5: Tuning large neural networks via zero-shot hyperparameter transfer.
\newblock \emph{arXiv preprint arXiv:2203.03466}, 2022.
\newblock URL \url{https://arxiv.org/abs/2203.03466}.

\bibitem[Yang et~al.(2023)Yang, Yu, Zhu, and Hayou]{tp_6}
Greg Yang, Dingli Yu, Chen Zhu, and Soufiane Hayou.
\newblock Tensor programs vi: Feature learning in infinite-depth neural networks, 2023.
\newblock URL \url{https://arxiv.org/abs/2310.02244}.

\bibitem[Yang et~al.(2024)Yang, Simon, and Bernstein]{spectral_condition}
Greg Yang, James~B. Simon, and Jeremy Bernstein.
\newblock A spectral condition for feature learning, 2024.
\newblock URL \url{https://arxiv.org/abs/2310.17813}.

\bibitem[Zhou et~al.(2022)Zhou, Lei, Liu, Du, Huang, Zhao, Dai, Chen, Le, and Laudon]{zhou2022mixtureofexperts}
Yanqi Zhou, Tao Lei, Hanxiao Liu, Nan Du, Yanping Huang, Vincent Zhao, Andrew Dai, Zhifeng Chen, Quoc Le, and James Laudon.
\newblock Mixture-of-experts with expert choice routing, 2022.

\bibitem[Zoph et~al.(2022)Zoph, Bello, Kumar, Du, Huang, Dean, Shazeer, and Fedus]{zoph2022st}
Barret Zoph, Irwan Bello, Sameer Kumar, Nan Du, Yanping Huang, Jeff Dean, Noam Shazeer, and William Fedus.
\newblock St-moe: Designing stable and transferable sparse expert models.
\newblock \emph{arXiv preprint arXiv:2202.08906}, 2022.

\end{thebibliography}
\bibliographystyle{iclr2026_conference}

\newpage
\appendix

\section{MuP for Dense Transformer}\label{apdx:dense}

In this section, we verify the findings from \citet{yang2022tuning} by implementing the $\mu$P for dense models. In contrast to standard parameterization, for $\mu$P, the optimal learning rate transfers between different model widths.

\begin{figure}
    \centering
    \includegraphics[width=\linewidth]{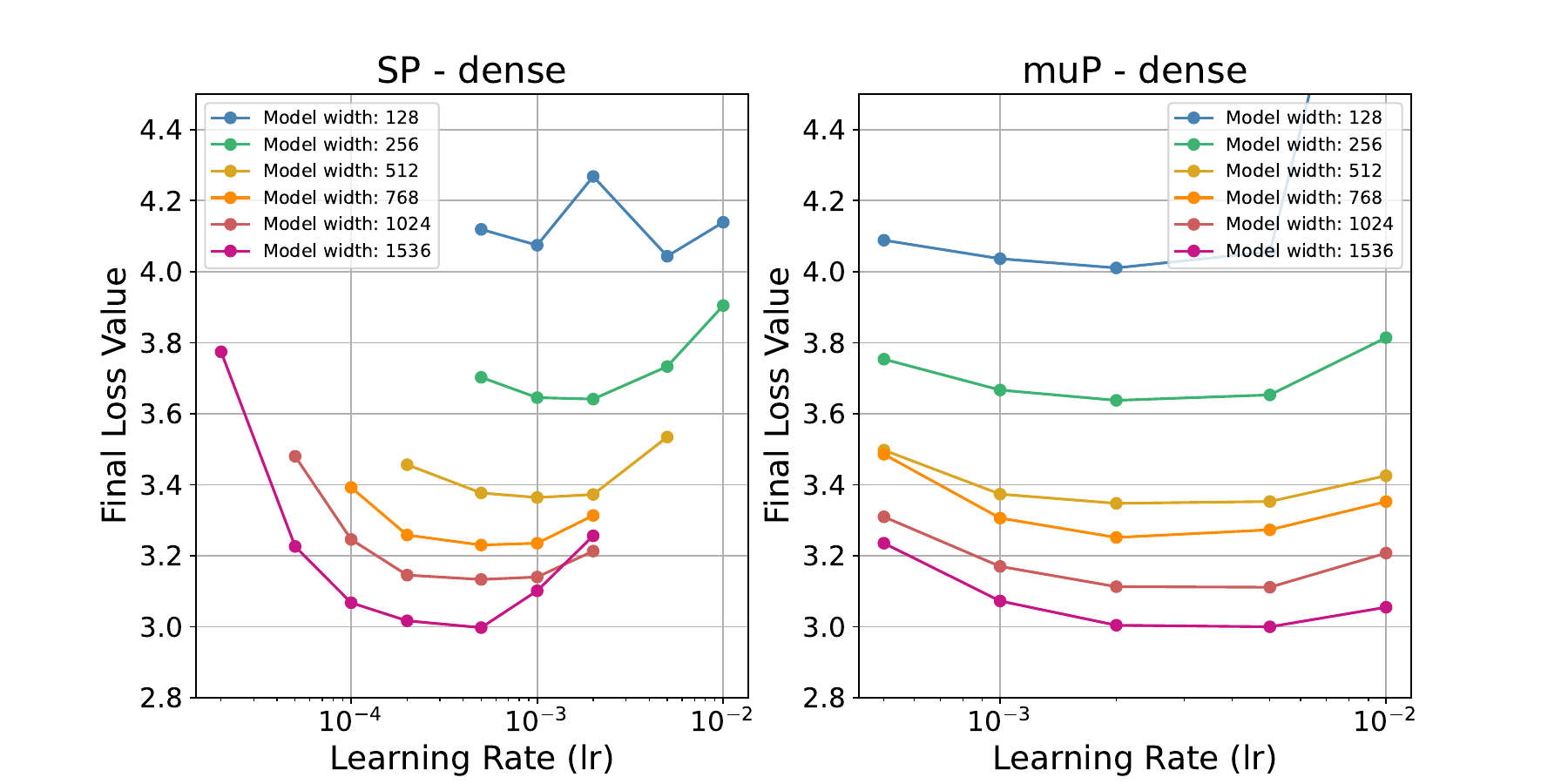}
    \label{fig:lr_plot_dense}
    \caption{This figure shows experiments on learning rate transfer in \textit{dense} models. SP on the left has a different optimal learning rate for each model width, while $\mu$P has a mostly stable optimum (with slight upward shift, same as TP5 and MoE sweeps).}
        \label{fig:mup_dense}
\end{figure}

\section{Expert–gradient covariance lemma}
\label{sec:eg-cov}

We now formalize the intuition that in a µP-parametrized Switch-MoE block, each active expert’s forward activation correlates with its backward gradient at order $\Theta(1/n)$, and that the router’s gradient norm remains $\Theta(1)$ both immediately after initialization and again after one µ-SGD/Adam update.

\begin{lemma}[Expert–gradient covariance]
\label{lem:eg-cov}
Let an $L$‐block Switch‐MoE be µP‐parametrized with width $n\to\infty$, a fixed number of experts $n_{\mathrm{experts}}=O(1)$, and fixed top-$k=O(1)$.  For each block $\ell$, define
\begin{equation}
\label{eq:block-def}
\begin{aligned}
y^{(\ell)} &= \sum_{e=1}^{n_{\mathrm{experts}}}
    R^{(\ell)}_e\bigl(x^{(\ell)}\bigr)\,
    E^{(\ell)}_e\bigl(x^{(\ell)}\bigr),\\
\delta^{(\ell)} &= \nabla_{y^{(\ell)}}L \in \R^n.
\end{aligned}
\end{equation}
Assume the inductive hypothesis
\begin{equation}
\label{eq:H_l}
\begin{aligned}
\frac1n\sum_{j=1}^n\bigl(\delta^{(\ell)}_j\bigr)^2
  &= \Theta(n^{-2})
  \;\Longrightarrow\;
  \delta^{(\ell)}_j = \Theta(n^{-1})
  \quad\text{for typical }j.
\end{aligned}
\tag{H$_\ell$}
\end{equation}
Then for every block $\ell$ and any active expert $e$, at both
\begin{equation*}
\label{eq:timepoints}
\begin{aligned}
t=0 &:\ \text{immediately after initialization},\\
t=1 &:\ \text{after one }\mu\text{-SGD/Adam step},
\end{aligned}
\end{equation*}
we have
\begin{equation}
\label{eq:lemma-result}
\begin{aligned}
\operatorname{Cov}\bigl(E^{(\ell)}_{e,j},\,\delta^{(\ell)}_j\bigr)
  &= \Theta(n^{-1}),\\
\bigl\|\nabla_{r^{(\ell)}}L\bigr\|_2
  &= \Theta(1).
\end{aligned}
\end{equation}
\end{lemma}

\begin{proof}
\textbf{Notation.}
For block $\ell$, set
\begin{equation}
\label{eq:notation}
\begin{aligned}
\mathbf e_e
  &= E^{(2)}_{\ell,e}\,
    \mathrm{ReLU}\!\bigl(E^{(1)}_{\ell,e}x^{(\ell)}\bigr)
    \in\R^n,\\
R_e      &= R^{(\ell)}_e\bigl(x^{(\ell)}\bigr)\in\R,\\
\delta   &= \delta^{(\ell)}\in\R^n,\\
J^{(\ell)}
  &= \frac{\partial\,y^{(\ell)}}{\partial\,x^{(\ell)}}
    \in\R^{n\times n},
\end{aligned}
\end{equation}
so that
\(\delta^{(\ell-1)}=(J^{(\ell)})^\top\,\delta\)
and \(J^{(\ell)}_{ij}\sim\mathcal N(0,1/n)\) under µP.

\medskip
\noindent\textbf{Step 1: Stein’s lemma (holds at $t=0$ and $t=1$).}
Fix expert $e$, coordinate $j$, and define
\begin{equation}
\label{eq:stein-setup}
\begin{aligned}
Z            &= \mathbf e_{e,j}\sim\mathcal N(0,\sigma^2),\\
c            &= \sum_{e'\neq e}R_{e'}\,\mathbf e_{e',j},\\
g(z)         &= [L'(y^{(\ell)})]_j,\\
y^{(\ell)}_j &= R_e\,z + c.
\end{aligned}
\end{equation}
Then \(g'(z)=R_e\,[L''(y^{(\ell)})]_j\) and by \eqref{eq:H_l}, \([L''(y^{(\ell)})]_j=\Theta(n^{-1})\).  Thus
\begin{equation}
\label{eq:stein}
\begin{aligned}
\operatorname{Cov}(Z,g(Z))
 &= \sigma^2\,\mathbb{E}\bigl[g'(Z)\bigr]\\
 &= \sigma^2\,R_e\,\mathbb{E}\bigl[L''(y^{(\ell)})\bigr]_j\\
 &= \Theta(1)\cdot\Theta(1)\cdot\Theta(n^{-1})
 = \Theta(n^{-1}).
\end{aligned}
\end{equation}

\noindent\textbf{Remark.} Because by induction the block-$(\ell+1)$ Hessian entries already scale like \(\Theta(n^{-1})\), and passing any such matrix back through a µP-linear layer (whose weights are \(\mathcal{N}(0,1/n)\)) multiplies each term by another \(1/n\) but sums over \(n\) of them, the net effect is still \(\Theta(n^{-1})\). In other words, a \(1/n\) factor per weight-matrix multiplication exactly preserves the \(\Theta(n^{-1})\) scale of \(\bigl[L''(y^{(\ell)})\bigr]_j\).

\medskip
\noindent\textbf{Step 2: Router‐gradient norm (holds at $t=0$ and $t=1$).}
Summing the covariances over $j$,
\begin{equation}
\label{eq:router-norm}
\begin{aligned}
\mathbb{E}\bigl[\mathbf e_e^\top \delta\bigr]
 &= \sum_{j=1}^n \operatorname{Cov}(\mathbf e_{e,j},\,\delta_j)
 = n\cdot\Theta(n^{-1})
 = \Theta(1).
\end{aligned}
\end{equation}
Since \(\operatorname{Var}(\mathbf e_e^\top\delta)=O(n^{-1})\), Chebyshev’s inequality gives
\(\mathbf e_e^\top\delta=\Theta(1)\) with high probability, i.e.\ the second line of \eqref{eq:lemma-result}.

\medskip
\noindent\textbf{Step 3: One‐step update.}
Under µ-SGD/Adam with LR $\eta/n$ on experts and $\eta$ on router, the factors in \eqref{eq:stein} and \eqref{eq:router-norm} change by at most a $(1+O(n^{-1}))$ factor, so both lines of \eqref{eq:lemma-result} hold at $t=1$.

\medskip
\noindent\textbf{Step 4: Depth induction.}
Using \(\delta^{(\ell-1)}=(J^{(\ell)})^\top\delta^{(\ell)}\) and \(J^{(\ell)}_{ij}\sim\mathcal N(0,1/n)\), one shows
\(\tfrac1n\sum_j(\delta^{(\ell-1)}_j)^2=\Theta(n^{-2})\), establishing \((H_{\ell-1})\).  The base case $\ell=L$ is given by Tensor-Programs V; induction completes the proof.
\end{proof}

\section{Experimental setup}
All models in this study are decoder-only Transformers trained on the C4 dataset~\citep{raffel2020exploring}. We use the GPT-2 tokenizer~\citep{radford2018improving} and optimize with AdamW~\citep{loshchilov2019decoupled}. Training follows a cosine decay schedule with linear warmup for the first $1\%$ of steps. Weights are initialized with a normal distribution, as the theory of Tensor Programs assumes \citep{yang2022tuning}. Mixed precision training is applied, with Attention component computed at high precision. The models employ MLP with ReLU activations. MoE models are Switch Transformers \citep{fedus2022switch}. As a standard MoE setup we used $8$ Experts, $1$ of which is activated per token. All models have Attention head dimension of $64$. Two auxiliary losses are used for the Router: a z-loss weighted at $0.001$~\citep{zoph2022st} and load balancing weighted at $0.01$~\citep{fedus2022switch}.

Models for Figure~\ref{fig:lr_plot_moe} are trained for $5.2$B tokens with $16$ Transformer layers ($16\times$ attention + feed-forward). Dense models for Figure~\ref{fig:mup_dense} have $24$ layers and are trained for $16$B tokens. Models for Figure~\ref{fig:other_moe_aspects} have $12$ blocks and are trained for $2.5$B tokens. Both dense and MoE width sweeps use a base width of $128$. In Figure~\ref{fig:lr_plot_moe}, the runs with width $128$ are reused across all three panels since, when model width equals the base width, all parameterizations coincide. Additionally, In Fig.~\ref{fig:lr_plot_moe_small} the models have 8 transformer layers and are trained for 1B tokens. 

Differences in total tokens and layer counts across figures arise from experiments conducted at different times under changing compute budgets. MoE runs are smaller/shorter to enable one additional width; we do not expect these scale differences to affect the qualitative conclusions.

The expert-count and granularity ablations are run under \textit{simple}P at base width $256$ and model width $768$. Because width is fixed within each sweep, the parameterization only induces a global shift in the learning-rate scale and does not affect relative comparisons; the same conclusions would hold under SP or $\mu$P.


\section{More results}

Fig.~\ref{fig:lr_plot_moe_small} presents the smaller models with 8 transformer layers and trained for 1B tokens. The results remain consistent with the findings from the main paper. The performance transfers under both \textit{simple}P and $\mu$P.

\begin{figure*}[h]
    \centering
    \includegraphics[width=\linewidth, trim=3.5cm 0 3.5cm 1cm, clip]{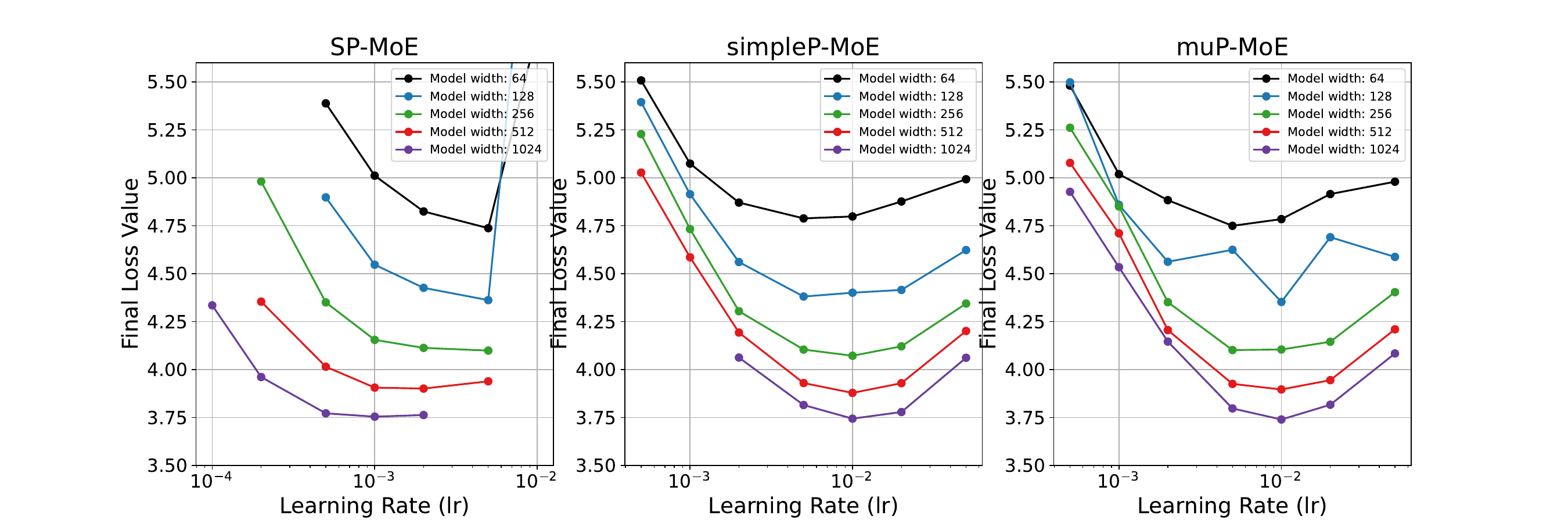}
    \caption{The plots present MoE performance for varying learning rates in the following set-ups: standard parametrization (SP) with no scaling on the left. \textit{simple}P - treating each Expert like a FeedForward layer in the middle. $\mu$P - our theory applied to the MoE layer on the right. While in the case of SP, the optimal learning rate varies with different model sizes, both reparameterizations achieve learning rate transfer across model widths.}
     \label{fig:lr_plot_moe_small}
     \vspace{-0.4cm}
\end{figure*}

\end{document}